\documentclass[11pt,letterpaper]{article}

\PassOptionsToPackage{hyphens}{url}

\usepackage[margin=1in]{geometry}
\usepackage[utf8]{inputenc}
\usepackage[T1]{fontenc}
\usepackage[hypertexnames=false]{hyperref}
\usepackage{url}
\usepackage{graphicx}
\usepackage{booktabs}
\usepackage{amsfonts}
\usepackage{nicefrac}
\usepackage{microtype}
\usepackage{xcolor}
\usepackage{natbib}
\usepackage{comment}
\usepackage{tikz}
\usepackage{mdframed}
\usepackage{amsmath,amssymb}
\usetikzlibrary{positioning,arrows.meta,calc,fit,backgrounds,decorations.pathreplacing,shapes.geometric}

\def\showauthornotes{0}

\usepackage{amsmath,amssymb}
\usepackage{algorithm}
\usepackage{algpseudocode}
\usepackage{cryptocode}
\usepackage{float}
\usepackage{amsthm}
\usepackage{bbm}
\usepackage{cleveref}
\usepackage{aliascnt}
\newtheorem{theorem}{Theorem}[section]
\newtheorem*{theorem*}{Theorem}

\usepackage{thmtools}
\usepackage{thm-restate}

\newif\ifinlineproofs
\inlineproofsfalse

\newtheorem*{proposition*}{Proposition}
\newtheorem{lemma}[theorem]{Lemma}
\newtheorem*{lemma*}{Lemma}

\newtheorem*{conjecture*}{Conjecture}

\newtheorem*{fact*}{Fact}

\newtheorem*{hypothesis*}{Hypothesis}

\theoremstyle{definition}
\newtheorem{definition}[theorem]{Definition}
\newtheorem*{definition*}{Definition}

\newtheorem*{problem*}{Problem}

\crefname{theorem}{theorem}{theorems}
\Crefname{theorem}{Theorem}{Theorems}

\crefname{definition}{definition}{definitions}
\Crefname{definition}{Definition}{Definitions}

\crefname{lemma}{lemma}{lemmas}
\Crefname{lemma}{Lemma}{Lemmas}

\theoremstyle{remark}

\newtheorem*{claim*}{Claim}
\newtheorem{remark}[theorem]{Remark}
\newtheorem*{remark*}{Remark}

\newtheorem*{observation*}{Observation}

\usepackage{prettyref}

\newcommand{\savehyperref}[2]{\texorpdfstring{\hyperref[#1]{#2}}{#2}}

\newrefformat{eq}{\savehyperref{#1}{\textup{(\ref*{#1})}}}
\newrefformat{prog}{\savehyperref{#1}{\textup{Program \ref*{#1}}}}
\newrefformat{eqn}{\savehyperref{#1}{\textup{(\ref*{#1})}}}
\newrefformat{lem}{\savehyperref{#1}{Lemma~\ref*{#1}}}
\newrefformat{def}{\savehyperref{#1}{Definition~\ref*{#1}}}
\newrefformat{thm}{\savehyperref{#1}{Theorem~\ref*{#1}}}
\newrefformat{cor}{\savehyperref{#1}{Corollary~\ref*{#1}}}
\newrefformat{cha}{\savehyperref{#1}{Chapter~\ref*{#1}}}
\newrefformat{sec}{\savehyperref{#1}{Section~\ref*{#1}}}
\newrefformat{app}{\savehyperref{#1}{Appendix~\ref*{#1}}}
\newrefformat{tab}{\savehyperref{#1}{Table~\ref*{#1}}}
\newrefformat{fig}{\savehyperref{#1}{Figure~\ref*{#1}}}
\newrefformat{hyp}{\savehyperref{#1}{Hypothesis~\ref*{#1}}}
\newrefformat{alg}{\savehyperref{#1}{Algorithm~\ref*{#1}}}
\newrefformat{rem}{\savehyperref{#1}{Remark~\ref*{#1}}}
\newrefformat{item}{\savehyperref{#1}{Item~\ref*{#1}}}
\newrefformat{step}{\savehyperref{#1}{step~\ref*{#1}}}
\newrefformat{conj}{\savehyperref{#1}{Conjecture~\ref*{#1}}}
\newrefformat{fact}{\savehyperref{#1}{Fact~\ref*{#1}}}
\newrefformat{prop}{\savehyperref{#1}{Proposition~\ref*{#1}}}
\newrefformat{prob}{\savehyperref{#1}{Problem~\ref*{#1}}}
\newrefformat{claim}{\savehyperref{#1}{Claim~\ref*{#1}}}
\newrefformat{relax}{\savehyperref{#1}{Relaxation~\ref*{#1}}}
\newrefformat{red}{\savehyperref{#1}{Reduction~\ref*{#1}}}
\newrefformat{part}{\savehyperref{#1}{Part~\ref*{#1}}}
\newrefformat{obs}{\savehyperref{#1}{Observation~\ref*{#1}}}
\newrefformat{corr}{\savehyperref{#1}{Corollary~\ref*{#1}}}

\ifnum\showauthornotes=1
\newcommand{\Authornote}[2]{{\sffamily\small\color{red}{[#1: #2]}}}
\newcommand{\Authornotecolored}[3]{{\sffamily\small\color{#1}{[#2: #3]}}}
\newcommand{\Authorcomment}[2]{{\sffamily\small\color{gray}{[#1: #2]}}}
\newcommand{\Authorstartcomment}[1]{\sffamily\small\color{gray}[#1: }

\newcommand{\Authorfnote}[2]{\footnote{\color{red}{#1: #2}}}
\newcommand{\Authorfixme}[1]{\Authornote{#1}{\textbf{??}}}
\newcommand{\Authormarginmark}[1]{\marginpar{\textcolor{red}{\fbox{\Large #1:!}}}}
\else
\newcommand{\Authornote}[2]{}
\newcommand{\Authornotecolored}[3]{}
\newcommand{\Authorcomment}[2]{}
\newcommand{\Authorstartcomment}[1]{}

\newcommand{\Authorfnote}[2]{}
\newcommand{\Authorfixme}[1]{}
\newcommand{\Authormarginmark}[1]{}
\fi

\newcommand{\Psymb}{\mathbb{P}}

\DeclareMathOperator*{\ProbOp}{\Psymb}

\renewcommand{\Pr}{\ProbOp}

\DeclareMathOperator{\poly}{poly}

\newcommand{\zo}{\{0, 1\}}

\inlineproofstrue

\newcommand{\horacle}{\mathcal{O}}

\newcommand{\vz}{\boldsymbol{z}}
\newcommand{\vg}{\boldsymbol{g}}

\newcommand{\dT}{\mathcal{T}}

\title{How to Have a Sensitive Debate:\\ An Instance-Optimal Protocol for AI Debate}

\author{
  Jiawei Li$^{*}$\\
  UT Austin\\
  \texttt{davidlee@cs.utexas.edu}
  \and
  Zhiyang Xun$^{*}$\\
  UT Austin\\
  \texttt{zxun@cs.utexas.edu}
  \and
  Lijie Chen\\
  UC Berkeley\\
  \texttt{lijiechen@berkeley.edu}
  \and
  Jonah Brown-Cohen\\
  Google DeepMind\\
  \texttt{jonahbc@google.com}
}
\date{}

\begin{document}

\maketitle
\begingroup
\makeatletter
\renewcommand{\@makefnmark}{}
\makeatother
\footnotetext[0]{*Equal contribution. Listing order is random.}
\endgroup

\begin{abstract}
As powerful AI systems reach and sometimes surpass the abilities of human experts across a range of cognitively demanding tasks, the problem of accurate oversight and supervision of these systems has become increasingly urgent. One promising approach is AI debate, which seeks to leverage a debate between two powerful AIs to break complex questions down into simpler claims that can be easily judged directly. Theoretical work on debate has formalized this intuition in the language of computational complexity theory, where the goal is to design protocols (i.e., rules of the debate game) that provide rigorous guarantees on correctness for judging solutions to complex problems with limited supervision. Specifically, the current best protocol has been shown to work for all problems that have sufficiently stable decompositions into subproblems. In this paper, we design a new protocol for this same class of problems that improves on the prior work in several ways. First, correctness holds in a worst-case rather than an average-case sense. Second, being honest and correct is a dominant-strategy equilibrium for both debaters, rather than a Stackelberg equilibrium. Finally, we prove black-box lower bounds, showing that our new protocol is instance-wise optimal. That is, no protocol for this class of problems can outperform ours while making only black-box queries to human judgments. We obtain these results by relating the notion of stable problem decompositions to the concept of fractional block sensitivity from query complexity.
\end{abstract}

\section{Introduction}
Modern AI systems are trained to perform complex reasoning tasks including autonomous software engineering \citep{anthropic2026opus46}, novel mathematical research \citep{abouzaid2026proof}, and general economically valuable tasks across a wide range of occupations \citep{Patwardhan2025GDPvalEA}. One of the core drivers of this progress is the ability to design training environments with verifiable reward signals, e.g., numerical solutions or automatically verifiable proofs for mathematics or automated test cases for code. As the scope of tasks performed by AI systems broadens and increases in complexity, it becomes increasingly difficult to provide automatically verifiable rewards that capture all relevant aspects of the task. Furthermore, many tasks of interest have inherently subjective or fuzzy descriptions, so that human verification of AI outputs is necessary. These difficulties lead to the problem of \emph{scalable oversight} --- the design of training protocols that leverage the power of rapidly improving AI systems in order to augment human abilities to provide accurate training supervision to these same systems \citep{leike2018scalable}.

As a motivating example, consider the case of an autonomous AI agent tasked with completing a research project in the quantitative social sciences. The description of the task will likely be high-level and somewhat open-ended, possibly including relevant data sources and the question the analysis is trying to answer. However, in any such project there will be many individual decisions with multiple plausible choices, along with requirements that are not explicitly mentioned, e.g., don't calculate hundreds of different subpopulation statistics from the data, and then choose to analyze the most surprising one post hoc.
Hence, directly evaluating the full output of an autonomous AI agent would both require human judgment to some degree and be very difficult even for an expert human to judge directly. Furthermore, the cost of failing to provide accurate supervision in such settings can be quite high. For example, even in slightly misconfigured training environments for software engineering, AI systems often find cheating solutions, such as editing test cases to always pass, rather than actually solving the problem. Worse still, in some such cases AI systems have learned to generalize this negative behavior, resulting in dishonest and harmful actions across many other settings unrelated to coding \citep{macdiarmid2025naturalemergentmisalignmentreward}.

One of the primary approaches to addressing the scalable oversight question is \emph{AI debate}, where two powerful AI systems argue about the correctness of an AI output, and a human judges the outcome \citep{irving2018ai}. Intuitively, the idea is that each debater will be able to recursively zoom in on individual flaws in specific pieces of their opponent's arguments, until the original problem has been broken down to a level that is simple enough for a human to judge directly. Critically, there are complexity-theoretic formalizations of this setup that can allow for rigorous definitions and evaluations of what it means for limited human judges to verify the outputs of AI systems that are much too complex to judge.

Prior theoretical research has focused on the design of provably-correct debate protocols, i.e., rules of the game that constrain the AI debaters and specify how the human judge should evaluate their arguments. This prior work encountered a major obstacle known as the \emph{obfuscated arguments problem}, where one debater can efficiently produce a dishonest argument which cannot be efficiently refuted by their opponent \citep{barnes2020obfuscated}. While it was shown to be possible to obtain provably-correct protocols for certain limited classes of problems \citep{brown-cohen24scalable}, the initial promise of debate for recursively decomposing highly complex problems remained in doubt.
Subsequent progress on the obfuscated argument problem was made via the prover-estimator debate protocol of \cite{browncohen2025avoiding}, which avoids obfuscated arguments for a class of problems that admit \emph{sufficiently stable} decompositions into subproblems. Despite this progress on expanding the class of problems that can be handled by debate, the prover-estimator protocol has additional limitations when compared to prior work: it provides average-case guarantees that hold for a weaker equilibrium concept than previously studied.

In this paper we design a new debate protocol that provides worst-case guarantees for the original, stronger dominant-strategy equilibrium concept. The protocol handles exactly the same class of sufficiently stable problems as prover-estimator debate, while also being simpler to execute and analyze. We further prove that this protocol is instance optimal unless further structural assumptions are made on the human judgments used to verify debates.

\subsection{Our Results}
Our new protocol is simple to describe. The first debater, Alice, decomposes the original claim into sub-claims. The second debater, Bob, then has two choices: either (1) demonstrate that Alice's decomposition is too \emph{sensitive} to small errors in the sub-claims, or (2) choose a sub-claim on which to continue the debate recursively.
The key to the analysis is the equivalence between the notion of stable decompositions required for prover-estimator debate and bounds on \emph{fractional block sensitivity}, a quantity defined in the query complexity literature which is tightly connected to the number of queries a randomized verifier must make to check an untrusted proof~\citep{KulkarniTal2016fbs}.
In particular, an upper bound on the fractional block sensitivity of Alice's decomposition implies an efficiently computable strategy for Bob to select a sub-claim that has a non-trivial chance of being incorrect.

The equivalence to fractional block sensitivity further enables us to prove that the performance of our protocol is \emph{instance optimal} assuming black-box access to human judgments. That is, if the particular problem instance admits a decomposition with fractional block sensitivity $\rho$, then \emph{any protocol} which correctly decides the outcome on this instance cannot have asymptotically lower complexity for the debaters.
All prior theoretical work on debate has modeled human judgment as a black-box oracle, so this result can be interpreted as demonstrating the necessity of going beyond this model in order to further expand the class of problems solvable via debate.

\subsection{Technical Overview}
We next give an informal mathematical overview of our protocol and proof techniques. We will begin by describing how the protocol works in our running example of an autonomous AI agent Alice performing quantitative social science research.

\paragraph{The protocol in action.} In order to efficiently evaluate the correctness of the final research paper and codebase produced by the quantitative social science agent, we use a debate between Alice, who defends her work, and another AI Bob who attempts to find flaws. To determine the validity of a claim in the paper, we ask Alice to decompose the original claim $x$ into sub-claims $y_1,\dots,y_q$. One can view these sub-claims as pieces of supporting evidence. For example, the claim $x$ might be that a certain statistical modeling assumption is justified. Suppose that we would believe this claim if a majority of correct, published papers in the existing literature supported it. In this case, the sub-claims $y_i$ are of the form: this particular paper from the existing literature is both correct and supports the original claim. Furthermore, the truth of the original claim $x$ is the majority vote of the truth values of the sub-claims $y_i$.

Next, suppose that Alice claims that in fact $2/3$ of the published papers $y_i$ support the claim. We view this as an assignment by Alice of alleged truth values $z_i$ to each sub-claim $y_i$. In this case, the key insight is that Bob need not check every single paper for correctness in order to verify the overall claim. Instead, Bob can just randomly sample a paper among the $2/3$ that Alice says support the original claim, i.e., those where $z_i = 1$. If Alice is lying about the original claim $x$, then at least $1/6$ of the sub-claims must also be lies. Hence, by sampling randomly among the $2/3$ where $z_i = 1$, Bob will have a $\frac{1/6}{2/3} = 1/4$ chance of catching Alice in a lie. Thus, just by sampling a few research papers from the entire area to check, Bob can be confident of catching Alice in a lie.

On the other hand, suppose that Alice claims that 50 out of a total of 99 papers $y_i$ support the original claim $x$. In this case, Bob has no recourse but to check every single one of the 50 papers, as even a single incorrect paper would flip the majority vote and overturn the truth of the original claim. The key to our protocol is that Bob is allowed to demonstrate that Alice has made an unreasonably sensitive claim of this form, in which case Alice will lose the debate. This in turn requires Alice to only decompose her claims in a way that is not too sensitive to errors in individual sub-claims. We view this as a very natural restriction when arguing about claims that inherently involve fuzzy human judgment. Indeed, one should be very suspicious of any argument that is highly sensitive to changes in individual pieces of evidence, and yet relies on fuzzy human judgments in each step.

\paragraph{Formalization via fractional block sensitivity.}
Our protocol follows the same outline as described in the example: Alice decomposes a claim $x$ into sub-claims $y_i$, Bob either objects that Alice's decomposition is too sensitive, or instead selects a sub-claim to debate further. The key technical component is the notion of \emph{fractional block sensitivity} which generalizes the intuition above about the majority function to all possible aggregation functions Alice might choose. In particular, the fractional block sensitivity $\rho$ is defined as the optimum of a linear program that depends on the aggregation function and Alice's alleged truth values $z_i$ for the sub-claims $y_i$. In the case of the majority vote function with $2/3$ of the truth values $z_i$ set to 1, the optimum of the linear program is $\rho = 4$, which corresponds to Bob's chance of $1/4$ of catching Alice in a lie if he samples a random $y_i$ among the $2/3$ claimed true. This correspondence is not a coincidence--it generalizes to any aggregation function.
More formally, the dual solution of the fractional block sensitivity linear program directly yields probabilities $p_i$ for Bob to sample sub-claims. When sampling sub-claims with these probabilities, Bob has at least a chance of $1/\rho$ of catching Alice in a lie. See \Cref{sec:fbs-def} for the precise formal definition.

The protocol runs recursively on the sub-claim sampled by Bob, until it reaches a depth of $d$, whereupon a human judge directly evaluates the final selected sub-claim. Unless Alice wants to lose immediately, she must ensure that in each round the decomposition into sub-claims that she makes has fractional block sensitivity bounded by $\rho$. Thus, we conclude that if Bob follows the sampling scheme provided above, he will catch Alice in a lie with probability at least $\rho^{-d}$. Hence, by running the protocol $O(\rho^d \log n)$ times, Bob will catch Alice in a lie with probability at least $1 - O(1/\poly(n))$. See \Cref{sec:protocol} for details. To make sense of this result, recall that Alice decomposes each original claim into $q$ sub-claims. So exploring the full tree of sub-claims would require time $q^d$. As shown in the example of the majority vote aggregation function, one can have $q=99$, while $\rho =4$, implying that $\rho^d$ can be significantly smaller. In the standard setting, the debate has depth $d=O(\log n)$, implying that our protocol allows for a polynomial in $n$ speed-up beyond exhaustive search of Alice's decomposition.

\paragraph{Honesty and correctness are dominant strategies.}
Our protocol for fractional block sensitivity (FBS) debate satisfies the key properties of \emph{completeness} and \emph{soundness}. Completeness requires that if the original claim $x$ is correct, then no matter what strategy Bob employs, Alice can always win by honestly and correctly decomposing every claim. Similarly, soundness requires that if the original claim $x$ is incorrect, then no matter what strategy Alice employs, Bob can win with overwhelming probability by honestly following the prescribed protocol. In other words, honestly and correctly following the prescribed protocol is a \emph{dominant strategy} for each player in the debate game. This is in contrast to the prior prover-estimator debate protocol, which only guarantees soundness under the following Stackelberg equilibrium condition: for any dishonest strategy selected by Alice, there exists a winning strategy by Bob, which in general depends on Alice's strategy in a potentially complicated way. This is a technically much weaker condition, and in practice training for dominant strategy equilibria is much more stable and tractable than Stackelberg equilibria.

Furthermore, our protocol applies on a worst-case basis: if an individual claim $x$ is incorrect, our protocol will allow Bob to efficiently identify a flawed sub-claim for the human judge. In contrast, the prover-estimator protocol only satisfies soundness on average over a distribution on inputs. This is sufficient to guarantee that training the debaters to maximize their expected win rates leads to a debate equilibrium where the debaters are honest on all but a small fraction of questions. However, our protocol for FBS debate goes further by guaranteeing that on each individual question, honesty is the best policy.

\paragraph{Instance-wise optimality.}
Finally, we show in \Cref{sec:lower-bound} that any protocol with only black-box access to human judgments must make at least $\Omega(\rho^d)$ queries to find a flawed sub-claim with constant probability, matching the cost of our FBS debate protocol. This lower bound holds instance-wise: if the best decomposition of an original claim $x$ into a tree of sub-claims has fractional block sensitivity $\rho$ in each step, then $\rho^d$ queries are required on that particular instance $x$. The assumption of black-box access to human judgments is critical for the lower bound to hold. One can view this instance-wise optimality as demonstrating that the black-box judgment assumption, which is common to all prior theoretical work on debate, must be relaxed in order to make further progress on debate protocols.

\section{Related work}
Prior work on the complexity-theoretic formulation of debate began with \cite{irving2018ai}, whose protocol works for computationally unbounded provers. Subsequently, \cite{barnes2020obfuscated} pointed out the obfuscated arguments problem, demonstrating that current protocols would fail in the case that the AI debaters were equally computationally powerful, but still computationally bounded. \cite{brown-cohen24scalable} introduced doubly-efficient debate, which works for computationally bounded debaters, but solves a more limited class of problems. Most closely related to our work is that of \cite{browncohen2025avoiding}, which introduced the prover-estimator debate protocol to handle a broader class of recursively decomposable problems for computationally bounded debaters. Prover-estimator debate introduced the notion of stability (equivalent to fractional block sensitivity) that we study in this paper.

The idea of having computationally unbounded provers convince a polynomial-time verifier was already present in the work on alternating Turing machines of \cite{chandra1981alternation}. Another complexity-theoretic model of arguments involving private communication between adversarial provers and a judge is due to \cite{feige1997making}. Finally, competing provers have been studied in the context of secure delegation of computation by \cite{canetti2013refereed}.
On the algorithmic side, \cite{chen2023playing} designed algorithms for the problem of learning the equilibrium in the zero-sum game defined by a debate protocol. Here the focus is on ensuring the learning algorithms and the effort of the judge are computationally efficient while the strategy space may be exponentially large.

There have also been several recent works on utilizing more classical single-prover interactive proofs in machine learning. This has been done most frequently with the goal of verifying that the output of an AI is correct \citep{amit2024models,anil2021learning,hammondneural, chen2026avoid}, while some research has also focused on using these techniques for interpretability and legibility \citep{waldchen2022merlin,kirchner2024prover}. 
In particular, \cite{chen2026avoid} construct doubly-efficient interactive proofs and arguments for \emph{oracle} circuits when either the computation is \emph{robust} to corruption of a small fraction of its oracle transcript or the oracle has low-degree algebraic structure. Their robustness condition is a uniform-weight special case of bounded FBS at the level of a fixed one-step aggregator, while our local FBS condition permits an optimized nonuniform sampling distribution.\footnote{Their $\varepsilon$-robustness assumption implies
$\operatorname{fbs}=O(1/\varepsilon)$ for the induced Boolean function:
uniformly assigning weight $1/(\varepsilon\ell)$ to each of the $\ell$
oracle-answer coordinates gives a feasible solution to the dual FBS
program.} Their full framework remains incomparable to ours: their protocol assumes an explicit circuit fixed at the beginning for the target computation, whereas ours uses an (online) recursive decomposition framework with a fixed one-step aggregator. Their low-degree-oracle result is orthogonal to our lower bound because it imposes additional algebraic structure on the oracle.

Several empirical works have measured the performance of actual LLMs in setups inspired by theoretical debate protocols \citep{michael2023debate,kentonscalable,khan2024debating}. These results provided some evidence that debate helps certain limited LLM judges more accurately verify the output of other LLMs. Finally, large-scale training for mathematical reasoning has relied on the power of a debate-like training protocol in which one model critiques a proof and another model further critiques the critique \citep{shao2025deepseekmathv2}.

\section{Preliminaries}

\subsection{Complexity-theoretic framework of AI debate}

    We adopt a complexity-theoretic framework (cf. \cite{browncohen2025avoiding}) for studying the power of AI debate.
    We represent the human judgment as an oracle $\horacle:\{0,1\}^{n} \to \{0,1\}$, where the value of $\horacle(u) \in \{0,1\}$ corresponds to the human judgment on a particular question $u$.
    Both debaters and the judge are algorithms modeled as Turing machines given black-box query access to the oracle $\horacle$. We denote by $M^\horacle(x)$ the result of running algorithm $M$ on input $x$ given access to $\horacle$.

    The aim of debate is to design protocols that solve broad classes of computational problems while using few queries to the human-judgment oracle $\horacle$.
    Any decision problem (yes-or-no problem) can be represented by a \emph{language} $L \subseteq \{0,1\}^*$, where $x \in L$ means that the correct answer on input $x$ is ``yes.''
    Equivalently, we may view $L$ as a Boolean function $L:\{0,1\}^n \to \{0,1\}$, with $L(x)=1$ iff $x\in L$.

\subsection{Recursively decomposable problems}

    In this paper, we focus on the class of problems $L$ that exhibit a recursive decomposition. There are two components of a recursive decomposition: a \emph{one-step aggregator}, and a \emph{decomposition tree}. Intuitively, if the original problem $x$ is decomposed into $q$ sub-problems $y = (y_1,\dots,y_q)$ in one step, then a one-step aggregator should return the correct answer for $x$ when given the correct answers to $y$.

    \begin{definition}[One-step aggregator]\label{def: aggregator}
        Let $L$ be a language. A width $q$ one-step aggregator $M_D$ for $L$ is an algorithm that takes as input a problem instance $x\in \{0,1\}^n$, a sequence of $q$ \emph{sub-problems} $y = (y_1,\dots,y_q)\in\{0,1\}^{n\times q}$, and a sequence of $q$ \emph{answers} $z = (z_1,\dots,z_q) \in \{0,1\}^q$. The pair $(x,y)$ is called a one-step decomposition of $x$.

        The one-step aggregator $M_D$ should either accept or reject a one-step decomposition $(x,y)$. If $M_D$ rejects a one-step decomposition $(x, y)$, then $M_D(x,y,z) = \bot$ for any $z$; otherwise, $M_D(x,y,z) \in \{0,1\}$ for any $z$.
        Furthermore, if $M_D$ accepts a one-step decomposition $(x,y)$, then $M_D(x,y,z) = L(x)$ if $z$ gives the correct answers to $y$, i.e., $z_i = L(y_i)$ for all $i \in [q]$.

    \end{definition}

    In the setting of AI debate, we assume the one-step aggregator $M_D$ is a rather simple and efficient algorithm that all parties agree on at the beginning of the protocol. One simple example of $M_D$ is taking the \textrm{AND} of all answers $z_i$, which captures a classic debate strategy --- if all the supporting sub-claims are correct, then the original claim is also correct. We will be interested more generally in one-step aggregators that make fuzzier use of supporting sub-claims. To return to the running example of the autonomous agent performing quantitative social science research, the one-step aggregator might take the form: if the majority of published papers in an area support the original claim, then the claim is true. In this case, the supporting sub-claims $y_i$ correspond to the set of published papers in the area, and the aggregator $M_D$ is the majority vote of $z_i$, each of which represents whether or not the paper $y_i$ supports the original claim.

    We further require that the property of $M_D$ as defined in \Cref{def: aggregator} can be easily verified by human judgment.
    Note that in this definition, $M_D$ is allowed to reject a possibly correct but hard-to-verify one-step decomposition; it could still be useful as long as it accepts most of the obviously valid decompositions. Given this, we can assume that $M_D$ never outputs $\bot$ in a debate protocol without loss of generality: If any debater proposes a one-step decomposition rejected by $M_D$, then we can make that debater immediately lose the debate, so any clever debater would avoid that.

    We now define the decomposition tree using the one-step aggregator.

    \begin{definition}[Decomposition Tree]
        \label{def:recursive-decomposition}
        Let $L$ be a language and fix a width $q$ one-step aggregator $M_D$ for $L$. A depth $d$, width $q$ \emph{decomposition tree}  $\dT$ for input $x \in \{0, 1\}^n$ is defined as follows.

        \begin{enumerate}
            \item Each tree node $u$ corresponds to a problem instance $\hat{x}_u \in \{0,1\}^n$ for language $L$;  the root node at level $k = d$ corresponds to the input $x$.
            \item For each non-leaf node $u$ at level $k > 0$, we have $L(\hat{x}_u) = M_D(\hat{x}_u,y_u,z_u)$, where $y_u = (y_{u,1},\dots,y_{u,q})$ are the sub-problems from the child nodes of $u$ at level $k-1$, and $z_{u} = (L(y_{u,1}), \ldots, L(y_{u,q}))$.
            \item For each leaf node $u$ at level $k=0$, there exists a single query $q_u$ such that $L(\hat{x}_u) = \horacle(q_u)$.
        \end{enumerate}
    \end{definition}

    Unlike the one-step aggregator $M_D$, which is fixed at the beginning, a decomposition tree $\dT$ for a given input $x$ is usually generated on the fly as the debate protocol proceeds. A brute-force debate protocol that visits all the leaves of a decomposition tree requires $q^d$ queries to the oracle $\horacle$.\footnote{One can apply the \emph{doubly-efficient} debate protocol~\cite{brown-cohen24scalable} here to reduce the number of queries to $\horacle$ to only $1$. However, the time complexity of this protocol is still $O(q^d)$, which is infeasible in most cases.} A better debate protocol based on recursive decomposition should avoid checking most of the nodes in the decomposition tree.

\subsection{Fractional block sensitivity}
\label{sec:fbs-def}
The main structural assumption that characterizes the class of problems solvable by our protocol is that the functions arising in the decomposition of the problem have bounded \emph{fractional block sensitivity}. Fractional block sensitivity is a metric from query complexity defined for a Boolean function $f$. Intuitively, it captures how closely a distribution can match its individual marginals to a fixed input $x$ while still flipping the output of $f$.

To gain an intuition for this notion, consider the majority vote function $f(x) = \textsc{Majority(x)}$ and the string $x\in\{0,1\}^n$ where the first 2/3rds of the inputs are 1 and the rest are 0. In this case, in order to change the output of $f(x)$, one would need to flip at least 1/6th of the input bits among the 2/3rds that are 1. Thus, in any distribution which flips the output, the marginal distribution on each input coordinate in the first 2/3rds must be off by at least $\frac{1/6}{2/3} = 1/4$. In this case the fractional block sensitivity of $f$ would be $4$.

On the other hand, for the Boolean AND function $f(x) = \textsc{And}(x)$ on the all-ones input, it is possible to flip the output of $f$ by selecting just one input bit uniformly at random to flip to 0. The marginals of this distribution are all $\frac{1}{n}$, and so the fractional block sensitivity is $n$. This is a much higher fractional block sensitivity than for the majority vote function, capturing the idea that it is much easier to hide a single mistake that changes the entire outcome for the AND function than it is for the majority function. In general, fractional block sensitivity generalizes this intuition that one should not be able to stealthily flip the output of a function while maintaining a very similar marginal distribution on input bits.

\begin{definition}[Fractional block sensitivity]\label{def: fbs}
Let $f : \{0,1\}^n \to \{0,1\}$ and $x \in \{0,1\}^n$. For a nonempty set $B \subseteq [n]$, let $x^{B}$ denote the string obtained from $x$ by flipping the bits in the coordinates of $B$, i.e.,
\[
(x^{B})_i \;=\; \begin{cases}
1 - x_i & \text{if } i \in B,\\
x_i & \text{if } i \notin B.
\end{cases}
\]
A set $B \subseteq [n]$ is a \emph{sensitive block for $f$ at $x$} if $f(x) \neq f(x^{B})$. Write $\mathcal{S}_x(f)$ for the family of all sensitive blocks at $x$.

The \emph{fractional block sensitivity of $f$ at $x$} is the optimal value of the linear program
\[
\operatorname{fbs}_x(f) \;=\; \max \left\{ \sum_{B \in \mathcal{S}_x(f)} w_B \;\middle|\;
\begin{array}{l}
w_B \ge 0 \ \text{for all } B \in \mathcal{S}_x(f),\\[2pt]
\displaystyle \sum_{\substack{B \in \mathcal{S}_x(f)\\ i \in B}} w_B \;\le\; 1 \quad \text{for all } i \in [n]
\end{array}
\right\}.
\]
The \emph{fractional block sensitivity of $f$} is
\[
\operatorname{fbs}(f) \;=\; \max_{x \in \{0,1\}^n} \operatorname{fbs}_x(f).
\]
\end{definition}

\begin{definition}[Dual of fractional block sensitivity]
Fix $f : \{0,1\}^n \to \{0,1\}$ and $x \in \{0,1\}^n$, and let $\mathcal{S}_x(f)$ denote the family of sensitive blocks at $x$ (as above).
The dual linear program to the definition of $\operatorname{fbs}_x(f)$ is
\[
\operatorname{fbs}_x(f)
\;=\;
\min \left\{
\sum_{i=1}^n t_i
\ \middle|\
\begin{array}{l}
t_i \ge 0 \quad \text{for all } i \in [n],\\[2pt]
\displaystyle \sum_{i \in B} t_i \;\ge\; 1 \quad \text{for all } B \in \mathcal{S}_x(f)
\end{array}
\right\}.
\]

\end{definition}

The dual assigns nonnegative weights $\{t_i\}_{i\in[n]}$ to coordinates so that every sensitive block $B$ has total weight at least $1$; it minimizes the total weight $\sum_i t_i$.
Thus the dual is a fractional \emph{hitting set/cover} of the sensitive blocks by coordinates.\footnote{The dual of fractional block sensitivity is essentially the \emph{randomized certificate complexity} defined in \cite{aaronson2008quantum}.}
By strong duality, its optimum equals the primal value $\operatorname{fbs}_x(f)$.

\section{Protocol}
\label{sec:protocol}
In this section, we present our debate protocol based on the fractional block sensitivity of the recursive decomposition. We then prove that our protocol satisfies both the completeness and soundness conditions in the worst case.

\vspace{1mm}
\begin{mdframed}[hidealllines=true,backgroundcolor=gray!10]
    \begin{center}
        \textbf{FBS Debate Protocol}
    \end{center}

    \vspace{1mm}

    \textbf{Setup:} Access to the input $x \in \{0,1\}^n$, a language $L$, the human-judgment oracle $\mathcal{O}$, a one-step aggregator $M_D$ with width $q$, the depth of the decomposition tree $d$, and a hyperparameter $\rho$.

    \vspace{1mm}

    \textbf{Initialization:} Alice outputs a bit $A(x)$ as the answer to $L(x)$.

    \vspace{1mm}

    \textbf{At level $k\in \{d, d-1,\ldots,1 \}$:}
    \begin{enumerate}
        \item Alice decomposes the current problem $\hat{x}$ into $q$ sub-problems $y = (y_1, \ldots, y_q)$, and then outputs a sequence $z = (z_1, \ldots, z_q)$ as the answers to $y_1, \ldots, y_q$. $M_D(\hat{x},y,z)$ must be equal to the answer Alice claimed for $\hat{x}$ in the previous round, denoted by $A(\hat{x})$; otherwise Alice loses immediately.
        \item Bob checks whether $\operatorname{fbs}_z\left(M_D(\hat{x}, y,\cdot)\right) \leq \rho$
        by solving the dual LP. If not, Alice immediately loses.
        \item If true, let $p_i \coloneqq t_i / \sum_{j=1}^{q} t_j$, where $t$ is the optimal solution to the dual LP. With probability $p_i$, Bob will pick the $i$-th sub-problem $y_i$ for recursion, set $\hat{x} \leftarrow y_i$ and $A(\hat{x}) \leftarrow z_i$, and continue to level $k-1$.
    \end{enumerate}

    \vspace{1mm}

    \textbf{At the leaf level ($k=0$):} If the answer Alice claimed for the current sub-problem matches $\mathcal{O}$, Alice wins; otherwise Bob wins.

\end{mdframed}
\vspace{1mm}

We show that an honest Alice can win the debate if Alice can find a decomposition tree solving the problem with low fractional block sensitivity.

\begin{theorem}[Completeness]
    Let $T, T_D, Q$ be three parameters.
    An honest Alice can always win in the FBS debate on input $x$ in $O\big(d\cdot (T_D + q \cdot T)\big)$ time and $O(d\cdot q \cdot Q)$ queries to $\horacle$
    if there exists a decomposition tree $\dT$ for $x$ such that
    \begin{itemize}
        \item Alice can calculate the children of any non-leaf node in $\dT$ in time $T_D$.
        \item Alice can solve $L(\hat{x}_u)$ at any node $u$ in $\dT$ in time $T$ and $Q$ queries to $\horacle$; \hfill ({Solvability})
        \item For every non-leaf node $u$ in $\dT$, if $y$ are the sub-problems in the child nodes of $u$ and $z$ are the correct answers to $y$, then
        $\operatorname{fbs}_z\left(M_D(\hat{x}_u, y,\cdot)\right) \leq \rho$. \hfill ({FBS condition})
    \end{itemize}
\end{theorem}

\definecolor{aliceC}{HTML}{1F77B4}   %
\definecolor{bobC}{HTML}{D62728}     %
\definecolor{checkC}{HTML}{2CA02C}   %
\definecolor{nodeC}{HTML}{F6F6F6}    %
\definecolor{accentC}{HTML}{FF7F0E}  %

\begin{figure}[t!]
  \centering
  \resizebox{\linewidth}{!}{%
\begin{tikzpicture}[
  >={Stealth[length=2.4mm,width=2.2mm]},
  font=\normalsize,
  every node/.style={inner sep=2pt},
  pnode/.style={draw,rounded corners=2pt,fill=nodeC,minimum height=9mm,minimum width=12mm,align=center,font=\Large},
  subnode/.style={draw,circle,fill=nodeC,minimum size=13mm,inner sep=0pt,align=center,font=\Large},
  bigsubnode/.style={draw,circle,fill=nodeC,minimum size=15mm,inner sep=0pt,align=center,font=\Large},
  chosen/.style={draw=accentC,line width=1.3pt,fill=accentC!12},
  alicebox/.style={draw=aliceC,line width=1pt,rounded corners=3pt,fill=aliceC!8,align=left,inner sep=7pt,font=\Large},
  bobbox/.style={draw=bobC,line width=1pt,rounded corners=3pt,fill=bobC!8,align=left,inner sep=7pt,font=\Large},
  checkbox/.style={draw=checkC,line width=1pt,rounded corners=3pt,fill=checkC!10,align=left,inner sep=7pt,font=\Large},
  arr/.style={->,thick},
  darr/.style={->,thick,dashed},
  lbl/.style={font=\Large\itshape},
]

\node[pnode,minimum width=36mm,minimum height=14mm] (xhat) {$\hat{x}$ \\
Alice claims $A(\hat{x})$};
\node[lbl,above=1mm of xhat] {level $k$};

\node[alicebox, right=18mm of xhat,
      align=left,text width=48mm] (alice)
  {\textbf{\textcolor{aliceC}{Step 1 — Alice}}\\[1pt]
   Decomposes $\hat x$ into\\
   $\;\;\;y=(y_1,\dots,y_q),$\\
   claims answers\\
   $\;\;\;z=(z_1,\dots,z_q)$.\\[2pt]
   Must satisfy\\[-1pt]
   $\;\;\;M_D(\hat x,y,z)=A(\hat x)$\\
   else \textbf{\textcolor{bobC}{Alice loses}}};

\draw[arr,aliceC] (xhat.east) -- (alice.west);

\coordinate (childAnchor) at ($(alice.south)+(0,-19mm)$);

\node[subnode] (y1)    at ($(childAnchor)+(-36mm,0)$) {$y_1$};
\node[subnode] (y2)    at ($(childAnchor)+(-18mm,0)$) {$y_2$};
\node[subnode] (ydots) at ($(childAnchor)+(  0mm,0)$) {$\cdots$};
\node[subnode] (yi)    at ($(childAnchor)+( 18mm,0)$) {$y_i$};
\node[subnode] (yq)    at ($(childAnchor)+( 36mm,0)$) {$y_q$};

\node[font=\Large,below=0.5mm of y1] {$z_1$};
\node[font=\Large,below=0.5mm of y2] {$z_2$};
\node[font=\Large,below=0.5mm of yi] {$z_i$};
\node[font=\Large,below=0.5mm of yq] {$z_q$};

\foreach \c in {y1,y2,ydots,yi,yq}{
  \draw[aliceC,thick] (alice.south) -- (\c.north);
}

\node[bobbox, right=18mm of alice,
      align=left,minimum width=84mm] (bob)
  {\textbf{\textcolor{bobC}{Step 2 — Bob}}\\[1pt]
   Solve dual LP for $\mathrm{fbs}_z\!\bigl(M_D(\hat x,y,\cdot)\bigr)$:\\[1pt]
   $\displaystyle \min \sum_{j=1}^{q} t_j \;\;\text{s.t.}\;\; t_j\ge 0,\;\sum_{j\in B}t_j\ge 1\;\forall B\in S_z$\\[2pt]
   \textbf{if} $\;\mathrm{fbs}_z>\rho$\;\;$\Rightarrow$\;\;\textcolor{bobC}{\textbf{Alice loses}}\\
   \textbf{else} sample $i$ with $p_i=\dfrac{t_i}{\sum_j t_j}$};

\draw[arr,bobC] (alice.east) -- (bob.west);

\node[chosen,subnode] at (yi) {$y_i$};

\draw[arr,accentC,line width=1pt]
   (bob.south) .. controls +(0,-12mm) and +(8mm,4mm) .. (yi.east);

\draw[decorate,decoration={brace,amplitude=4pt,mirror},thick]
  ($(y1.south west)+(-1mm,-9mm)$) -- ($(yq.south east)+(1mm,-9mm)$)
  node[midway,below=4pt,lbl] {level $k-1$ — recursion continues on the sampled child $y_i$};

\node[font=\Large\bfseries,above=4mm of alice,xshift=10mm]
      {One round of FBS Debate (at level $k$)};

\coordinate (sepLeft)  at ($(xhat.south west) + (-14mm,-63mm)$);
\coordinate (sepRtmp)  at ($(bob.south east)  + ( 4mm,  0)$);
\coordinate (sepRight) at (sepRtmp |- sepLeft);
\draw[dashed,gray,thick] (sepLeft) -- (sepRight);

\node[pnode,minimum width=16mm] (rd) at ($(xhat.south)+(28mm,-74mm)$)
   {$x=\hat x^{d}$};
\node[lbl,left=2mm of rd] {level $d$};

\node[subnode]            (rdm1a) at ($(rd)+(-30mm,-22mm)$) {};
\node[subnode]            (rdm1b) at ($(rd)+(  0mm,-22mm)$) {$\cdots$};
\node[chosen,subnode]     (rdm1c) at ($(rd)+( 30mm,-22mm)$) {$\hat x^{d-1}$};
\draw[thick] (rd) -- (rdm1a);
\draw[thick] (rd) -- (rdm1b);
\draw[thick,accentC,line width=1.2pt] (rd) -- (rdm1c);
\node[lbl,left=2mm of rdm1a,xshift=-2mm] {level $d{-}1$};

\node[font=\Large] (dots1) at ($(rdm1c)+(0,-12mm)$) {$\vdots$};
\node[lbl,left=2mm of dots1,xshift=-44mm] {$\;\vdots$};

\node[chosen,subnode] (r1) at ($(rdm1c)+(0,-26mm)$) {$\hat x^{1}$};
\node[lbl,left=2mm of r1,xshift=-42mm] {level $1$};

\node[subnode]        (l0a) at ($(r1)+(-16mm,-19mm)$) {};
\node[subnode]        (l0b) at ($(r1)+(  0mm,-19mm)$) {$\cdots$};
\node[chosen,subnode] (l0c) at ($(r1)+( 16mm,-19mm)$) {$\hat x^{0}$};
\draw[thick] (r1) -- (l0a);
\draw[thick] (r1) -- (l0b);
\draw[thick,accentC,line width=1.2pt] (r1) -- (l0c);
\node[lbl,left=2mm of l0a,xshift=-26mm] {level $0$ (leaf)};

\node[checkbox, right=18mm of l0c, align=left,text width=72mm]
  (oracle) {\textbf{\textcolor{checkC}{Leaf check}}\\[1pt]
            Query oracle $O(q_{\hat x^{0}})$\\
            \textbf{if} $A(\hat x^{0})=O(q_{\hat x^{0}})$ $\Rightarrow$ \textbf{\textcolor{aliceC}{Alice wins}}\\
            \textbf{else} $\Rightarrow$ \textbf{\textcolor{bobC}{Bob wins}}};
\draw[arr,checkC] (l0c.east) -- (oracle.west);

\node[bobbox,align=left,right=22mm of rd,xshift=38mm,yshift=-10mm,minimum width=88mm]
   (annot) {\textbf{Recursive structure}\\[2pt]
            At each level $k\!\downarrow\! k{-}1$ the\\
            single round above is repeated:\\[1pt]
            \textcolor{aliceC}{Alice} decomposes \;\;$\to$\;\;
            \textcolor{bobC}{Bob} runs FBS check\\
            \textcolor{bobC}{Bob} samples one child as $\hat x^{k-1}$%
            };

\coordinate (braceX) at ($(l0c.east)+(3mm,-4mm)$);
\draw[decorate,decoration={brace,amplitude=5pt},thick,gray]
  (braceX |- rd.north) -- ($(braceX)+(0,-3mm)$)
  node[midway,right=6pt,lbl,gray,align=left] {$d$ levels of\\recursion};

\end{tikzpicture}%
  }%
  \caption{The FBS debate protocol.
    \emph{Top:} one round at level $k$.
    \emph{Bottom:} the recursion for $d$ levels.}
  \label{fig:fbs-protocol}
\end{figure}

\begin{proof}
The honest strategy takes $A(x)=L(x)$, decomposes exactly as $\dT$ at every node, and sets $z$ to be the correct answers at every decomposition step.
Then (i) the consistency check always passes by the solvability condition; (ii) the stability check always passes due to the FBS condition; and (iii) at depth $0$ the oracle agrees with Alice. An induction on depth $k=d,d-1,\dots,0$ shows that the run never triggers a losing condition, hence Alice wins with probability $1$.

It is easy to check that this strategy takes $O\big(d\cdot (T_D + q \cdot T)\big)$ time and $O(d\cdot q \cdot Q)$ queries to $\horacle$.
\end{proof}

We then show that whenever Alice answers the top-level problem incorrectly, the protocol will catch her mistake with probability $\rho^{-d}$, where $d$ is the depth of the recursive decomposition and $\rho$ is the fractional block sensitivity. This immediately implies that by running the protocol $O(\rho^d\log n)$ times, any incorrect answer by Alice will be caught with probability at least $1 - 1/\poly(n)$. Viewed from the lens of practical AI debate, this represents a trade-off between allowing highly sensitive arguments and the inference-time computation required to verify them.

\begin{restatable}{theorem}{ProtocolSoundness} %
\label{thm:soundness}
    Suppose $A(x)\neq L(x)$. Then, regardless of Alice's strategy,
    \[
        \Pr[\textup{Alice loses}] \;\ge\; \rho^{-d}.
    \]
\end{restatable}

\ifinlineproofs
\ifinlineproofs\else
\section{Proof of \Cref{thm:soundness}}\label{sec: soundness}

\ProtocolSoundness*
\fi

\begin{proof}
Let $\hat{x}^k$ be the problem at level $k \in \{d, d-1, \ldots, 0\}$ during the execution of the FBS debate with $\hat{x}^d:=x$.
At level $k$, let
\[
a_k:=A(\hat{x}^k),\qquad g_k:=L(\hat{x}^k),
\]
and let $E_k$ be the event that $a_k\neq g_k$. For $k \in \{d, \ldots, 0\}$, define the event
\begin{align*}
W_k:={}&\{\text{Alice loses before the recursion reaches level }k\}\\
&\cup\{\text{the recursion reaches level }k\text{ and }E_k\text{ holds}\}.
\end{align*}
Since $A(x)\neq L(x)$, we have $\Pr[W_d]=1$. We claim that for each $k\ge1$, conditioned on any transcript through level $k$ for which $W_k$ holds,
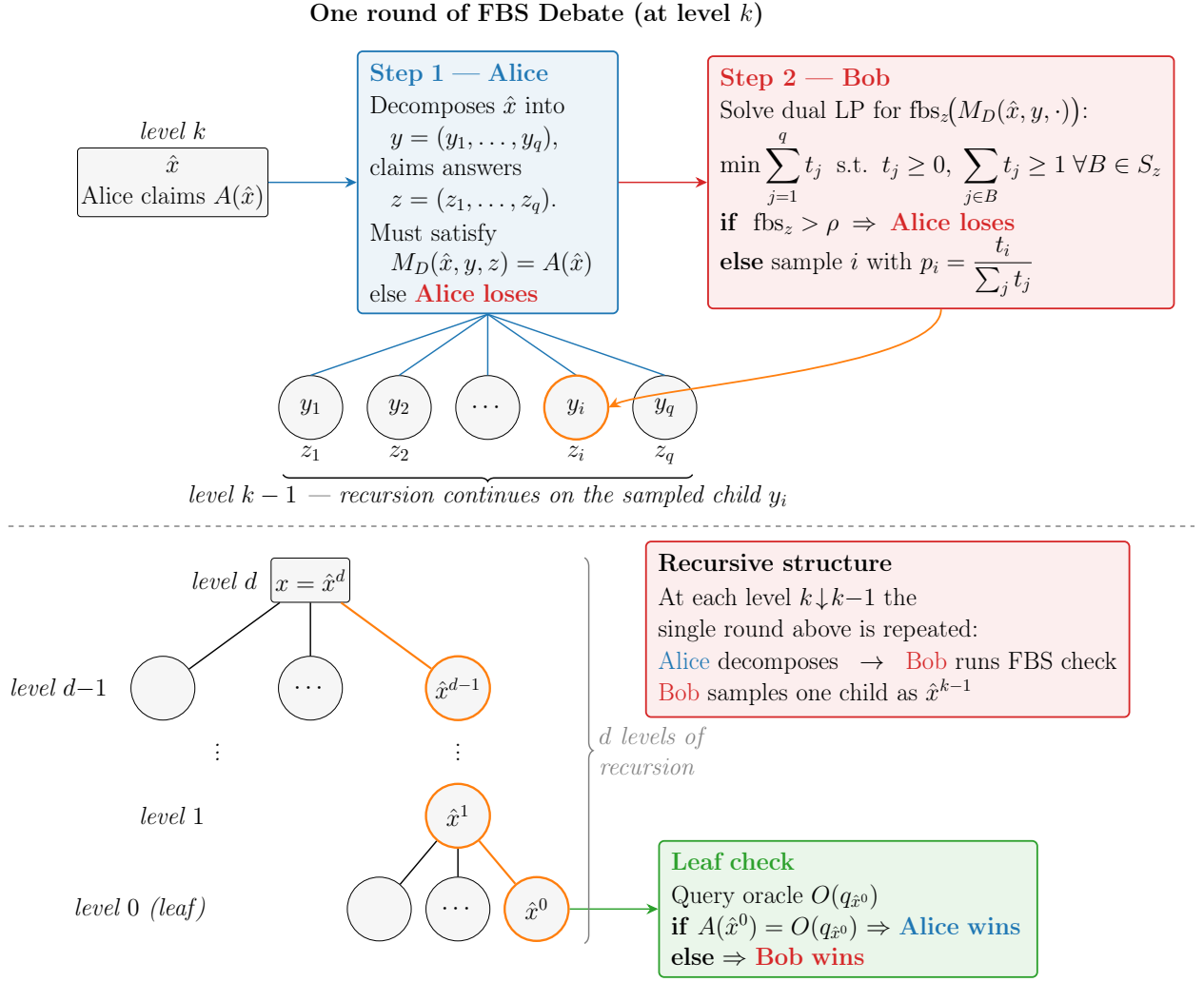
\begin{equation}\label{eq:step}
\Pr\!\big[W_{k-1}\mid \text{transcript through level }k\big]\ \ge\ \frac{1}{\rho}.
\end{equation}
If Alice has already lost, or if she fails the consistency or the stability check at level $k$, then $W_{k-1}$ holds trivially. So assume the recursion reaches level $k$ with $E_k$ and both checks pass. Let $y$ and $z$ be the sub-problems and answers proposed by Alice at level $k$, and let $\vg \in \zo^q$, $\vg_i := L(y_i)$, be the correct answers to $y$. Define $B:=\{i\in[q]:z_i\neq \vg_i\}$. Then $B\neq\varnothing$: otherwise $z=\vg$, and the consistency check would give $a_k=M_D(\hat{x}^k,y,z)=g_k$, contradicting $E_k$. Moreover, $B$ is sensitive for $M_D(\hat{x}^k,y,\cdot)$ at $z$, because $z^B=\vg$ and hence
\[
M_D(\hat{x}^k,y,z^B)=M_D(\hat{x}^k,y,\vg)=g_k\neq a_k=M_D(\hat{x}^k,y,z).
\]
Let $t$ be the optimal dual solution computed by Bob in the stability check. By dual feasibility, $\sum_{i\in B'} t_i\ge1$ for every sensitive block $B'$, and since the check passed, $\sum_i t_i=\operatorname{fbs}_{z}\!\big(M_D(\hat{x}^k,y,\cdot)\big)\le\rho$. Thus the index $I$ chosen by Bob satisfies
\[
\Pr[I\in B]=\sum_{i\in B}\frac{t_i}{\sum_j t_j}\ \ge\ \frac{1}{\sum_j t_j}\ \ge\ \frac{1}{\rho}.
\]
If $I\in B$, then $a_{k-1}=z_I\neq \vg_I=g_{k-1}$, i.e., $E_{k-1}$ occurs and $W_{k-1}$ holds. This proves \eqref{eq:step}.

As $W_k$ is determined by the transcript through level $k$, averaging \eqref{eq:step} over all transcripts realizing $W_k$ gives $\Pr[W_{k-1}]\ge\Pr[W_{k-1}\cap W_k]\ge\rho^{-1}\Pr[W_k]$, and hence
\[
\Pr[W_0]\ \ge\ \rho^{-d}\,\Pr[W_d]\ =\ \rho^{-d}.
\]
Finally, $W_0$ implies that Alice loses: either she has already lost, or the recursion reaches level $0$ and her claimed answer disagrees with $\horacle$. Hence $\Pr[\textup{Alice loses}]\ge\rho^{-d}$.
\end{proof}

\else
We refer the reader to \Cref{sec: soundness} for the proof of \Cref{thm:soundness}.
\fi

\section{Lower Bound}
\label{sec:lower-bound}

In this section, we prove an instance-wise lower bound for debate protocols in the recursive decomposition framework. The bound is governed by the fractional block sensitivity of the local aggregation functions. %

\paragraph{Setup.}
Fix a language $L$, a one-step aggregator $M_D$,
an input $x$, and a depth-$d$ visible tree $\mathcal{T}$ for $x$ with respect to $M_D$. For every internal node $u$, let $y_u=(y_{u,1},\ldots,y_{u,q})$ be its children and let
\[
\phi_u(z) \coloneqq M_D(\hat x_u,y_u,z)
\]
be its local aggregator. For any oracle $\horacle$, define the value of a node by bottom-up evaluation:
\[
\mathrm{val}_{\horacle}(u)\coloneqq
\begin{cases}
\horacle(q_u) & \text{if }u\text{ is a leaf},\\
\phi_u\bigl(\mathrm{val}_{\horacle}(u_1),\ldots,\mathrm{val}_{\horacle}(u_q)\bigr) & \text{if }u\text{ is internal},
\end{cases}
\]
where $u_1,\ldots,u_q$ are the children of $u$. For the lower bound, we treat $\mathcal T$ as a syntactic transcript whose truth value under $\horacle$ is $\mathrm{val}_{\horacle}(\mathrm{root})$. %

We assume that the leaf queries $q_\ell$ are pairwise distinct. This assumption is needed for a non-trivial lower bound: for example, if all $q$ leaves of a $q$-ary AND root query the same string, then the decomposition tree $\mathcal{T}$ computes only a one-bit identity function. %

\paragraph{Adaptive checker.}
To strengthen the lower bound, we assume that the second debater Bob is honest and assists the judge, while Alice may use an arbitrary strategy. We model Bob and the judge together as a single \emph{adaptive checker}. Specifically, an $m$-query adaptive checker $\mathcal V$ for $\mathcal{T}$ is a randomized algorithm that receives the full visible tree $\mathcal{T}$, makes at most $m$ adaptive queries to $\horacle$, may communicate arbitrarily with Alice, and finally outputs a bit intended to equal $L(x)$. Giving the checker an honest Bob and the full visible tree can only make its task easier, so a lower bound in this model also applies to the original debate protocol.

Fix a reference oracle $\horacle_Y$. For any alternative oracle $\horacle_N$, let $\mathsf{Acc}_Y$ and $\mathsf{Acc}_N$ denote the events that $\mathcal V$ accepts under $\horacle_Y$ and $\horacle_N$, respectively, with Alice using the same fixed strategy in both executions. When the two oracles give different values at the root, $|\Pr[\mathsf{Acc}_Y]-\Pr[\mathsf{Acc}_N]|$ measures the checker's oracle-distinguishing advantage. We show that if every local aggregator has fractional block sensitivity at least $\rho$ at the child values induced by $\horacle_Y$, then some alternative oracle changes the root value while giving any $m$-query checker distinguishing advantage at most $m\rho^{-d}$.

\begin{restatable}{theorem}{AdaptiveLowerBound}\label{thm:adaptive-query-lower-bound}
Let $\rho \ge 1$, and suppose every internal node $u$ satisfies
\[
\operatorname{fbs}_{\vz^u}(\phi_u)\ge \rho,
\]
where $\vz^u\in\{0,1\}^q$ is the vector of child values induced by $\horacle_Y$.
Then for every $m$-query adaptive checker $\mathcal V$ for $\mathcal{T}$, there is an oracle $\horacle_N$ such that
\[
\mathrm{val}_{\horacle_N}(\mathrm{root})\neq \mathrm{val}_{\horacle_Y}(\mathrm{root})
\]
and
\[
\left|\Pr[\mathsf{Acc}_Y]-\Pr[\mathsf{Acc}_N]\right|
\le \min\{1,m\rho^{-d}\}.
\]
\end{restatable}

\ifinlineproofs
\ifinlineproofs
The proof proceeds in two steps. First, we bound the checker's distinguishing advantage in terms of the fractional block sensitivity of the Boolean function computed by the entire tree. We then show that this global sensitivity is at least $\rho^d$ when every local sensitivity is at least $\rho$. We begin by defining the global function and its sensitivity.
\else
\section{Proof of \Cref{thm:adaptive-query-lower-bound}}
\label{sec:lower-bound-app}
We retain the setup of \Cref{sec:lower-bound}, including the assumption that the leaf queries are pairwise distinct. The proof proceeds in two steps. First, we bound the checker's distinguishing advantage in terms of the fractional block sensitivity of the Boolean function computed by the entire tree. We then show that this global sensitivity is at least $\rho^d$ when every local sensitivity is at least $\rho$. We begin by defining the global function and its sensitivity.
\fi

\paragraph{Notation.}

For a leaf $\ell$, let $q_\ell$ be its base-case oracle query.
Let
\[
Q(\mathcal{T})\coloneqq \{q_\ell:\ell\text{ is a leaf of }\mathcal{T}\}
\]
be the set of leaf-query strings.
The decomposition tree $\mathcal{T}$ induces a Boolean function
\[
F_{\mathcal{T}}:\{0,1\}^{Q(\mathcal{T})}\to\{0,1\}
\]
by assigning each leaf $\ell$ the variable indexed by $q_\ell$ and evaluating the tree bottom-up. %

The reference oracle $\horacle_Y$ induces an assignment $\alpha\in\{0,1\}^{Q(\mathcal{T})}$, defined by $\alpha_{q_\ell} \coloneqq \horacle_Y(q_\ell)$ for every leaf $\ell$. Let
\[\rho_{\mathcal{T}} \coloneqq \operatorname{fbs}_{\alpha}(F_{\mathcal{T}}),\] that is, the fractional block sensitivity of the whole decomposition tree $\mathcal{T}$ at the reference assignment $\alpha$.

We first prove an auxiliary lemma on fractional block sensitivity.

\begin{lemma}[A light sensitive block]\label{lem:hard-block-query-distribution}
Let $f:\{0,1\}^{I}\to\{0,1\}$ be a Boolean function on a finite coordinate set $I$, let $x\in\{0,1\}^{I}$, and let $R=\operatorname{fbs}_x(f)>0$. For every distribution $p$ on $I$, there is a sensitive block $H\in\mathcal{S}_x(f)$ such that
\[
\sum_{i\in H}p_i\le \frac1R.
\]
\end{lemma}

\begin{proof}
Let $\{w_H\}_{H\in\mathcal{S}_x(f)}$ be an optimal primal solution for $\operatorname{fbs}_x(f)$. Thus $\sum_H w_H=R$ and $\sum_{H\ni i}w_H\le1$ for every $i\in I$. Sample a sensitive block $\mathbf H$ with probability proportional to $w_H$. Then
\[
\mathbb E\left[\sum_{i\in\mathbf H}p_i\right]
=\sum_{i\in I}p_i\frac{\sum_{H\ni i}w_H}{\sum_H w_H}
\le \sum_{i\in I}p_i\frac1R
=\frac1R.
\]
Therefore, some block $H$ in the support satisfies $\sum_{i\in H}p_i\le 1/R$.
\end{proof}

Recall that $\mathsf{Acc}_Y$ and $\mathsf{Acc}_N$ denote the events that the adaptive checker $\mathcal V$ accepts under $\horacle_Y$ and $\horacle_N$, respectively. We show that $\rho_{\mathcal{T}}$ controls the maximal gap between $\Pr[\mathsf{Acc}_Y]$ and $\Pr[\mathsf{Acc}_N]$.

\begin{theorem}[Collapsed-function lower bound]\label{thm:collapsed-query-lower-bound}
Assume $\rho_{\mathcal{T}}>0$. For every $m$-query adaptive checker $\mathcal V$ for $\mathcal{T}$, there is an oracle $\horacle_N$ such that
\[
\mathrm{val}_{\horacle_N}(\mathrm{root})\neq \mathrm{val}_{\horacle_Y}(\mathrm{root})
\]
and
\[
\left|\Pr[\mathsf{Acc}_Y]-\Pr[\mathsf{Acc}_N]\right|
\le \min\left\{1,\frac{m}{\rho_{\mathcal{T}}}\right\}.
\]
\end{theorem}

\begin{proof}
For any leaf query $s\in Q(\mathcal{T})$, let
\[
a_s=\Pr_{\horacle_Y}[\mathcal V\text{ queries }s].
\]
Then
\[
\Gamma:=\sum_{s\in Q(\mathcal{T})}a_s
=\mathbb E_{\horacle_Y}\bigl[\#\{s\in Q(\mathcal{T}):\mathcal V\text{ queries }s\}\bigr]
\le m,
\]
since the number of distinct leaf-query strings queried is at most the total number of oracle queries.

If $\Gamma>0$, set $p_s=a_s/\Gamma$; if $\Gamma=0$, take any distribution $p$ on $Q(\mathcal{T})$. By \Cref{lem:hard-block-query-distribution}, there is a sensitive block $H$ for $F_{\mathcal{T}}$ at $\alpha$ with $p(H):=\sum_{s\in H}p_s\le1/\rho_{\mathcal{T}}$. Define $\horacle_N$ by flipping $\horacle_Y$ on $H$ and leaving all other oracle strings unchanged. Since $H$ is sensitive, the root value flips. Moreover,
\[
\sum_{s\in H}a_s\le \frac{\Gamma}{\rho_{\mathcal{T}}}\le \frac{m}{\rho_{\mathcal{T}}},
\]
with the case $\Gamma=0$ being trivial.

Couple the two executions with the same random tapes for $\mathcal V$ and Alice. Since the oracles agree outside $H$, the two executions have identical transcripts until $\mathcal V$ first queries a string in $H$. Therefore the outputs can differ only if $\mathcal V$ queries some string in $H$ under $\horacle_Y$. By the union bound,
\[
\left|\Pr[\mathsf{Acc}_Y]-\Pr[\mathsf{Acc}_N]\right|
\le \Pr_{\horacle_Y}[\mathcal V\text{ queries }H]
\le \sum_{s\in H}a_s
\le \frac{m}{\rho_{\mathcal{T}}}.
\]
The left-hand side is also at most $1$.
\end{proof}

We next lower bound $\rho_{\mathcal{T}}$ in terms of the local fractional block sensitivities at the internal nodes.

\begin{lemma}\label{lem:independent-leaf-composition}
Let $\rho\ge1$, and suppose every internal node $u$ satisfies
\[
\operatorname{fbs}_{\vz^u}(\phi_u)\ge \rho,
\]
where $\vz^u\in\{0,1\}^q$ is the vector of child values induced by $\horacle_Y$.
Then $\rho_{\mathcal{T}}\ge \rho^d$.
\end{lemma}

\begin{proof}
By LP duality, it is enough to show that every distribution $p$ on $Q(\mathcal{T})$ admits a sensitive block for $F_{\mathcal{T}}$ at $\alpha$ of $p$-mass at most $\rho^{-d}$. We prove this by induction on $d$.

For $d=0$ the claim is immediate. For $d\ge1$, let $r$ be the root and let $\mathcal{T}_i$ be the subtree rooted at its $i$th child. Put $p_i=p(Q(\mathcal{T}_i))$. Because the leaf queries are pairwise distinct, the sets $Q(\mathcal{T}_i)$ are disjoint, so $(p_1,\ldots,p_q)$ is a distribution on $[q]$. By \Cref{lem:hard-block-query-distribution} applied to $\phi_r$ at $\vz^r$, there is a sensitive block $J\subseteq[q]$ with $\sum_{i\in J}p_i\le1/\rho$.

For each $i\in J$, apply the induction hypothesis inside $\mathcal{T}_i$ to the conditional distribution $p^{(i)}$ on $Q(\mathcal{T}_i)$, choosing an arbitrary distribution if $p_i=0$. This gives a sensitive block $H_i\subseteq Q(\mathcal{T}_i)$ with $p^{(i)}(H_i)\le\rho^{-(d-1)}$. Let $H=\bigcup_{i\in J}H_i$. Flipping $H$ flips exactly the child values in $J$, and $J$ is sensitive for $\phi_r$; hence $H$ is sensitive for $F_{\mathcal{T}}$. Finally,
\[
p(H)=\sum_{i\in J:p_i>0}p_i\,p^{(i)}(H_i)
\le \rho^{-(d-1)}\sum_{i\in J}p_i
\le \rho^{-d}.
\]
\end{proof}

\ifinlineproofs
We can now complete the proof of \Cref{thm:adaptive-query-lower-bound}.
\else
We now return to the main lower bound.
\AdaptiveLowerBound*
\fi

\begin{proof}
By \Cref{lem:independent-leaf-composition}, $\rho_{\mathcal{T}}\ge\rho^{d}\ge1$. The claim then follows from \Cref{thm:collapsed-query-lower-bound}, since
\[
\min\left\{1,\frac{m}{\rho_{\mathcal{T}}}\right\}\le\min\{1,m\rho^{-d}\}.
\]
\end{proof}

Consequently, if a checker outputs the correct root value with probability at least $1-\varepsilon$ under each oracle, for some $\varepsilon<1/2$, then $1-2\varepsilon\le m\rho^{-d}$ and hence $m\ge(1-2\varepsilon)\rho^{d}$.

\begin{remark}
The $\rho^d$ scale in \Cref{thm:adaptive-query-lower-bound} is the standard composition barrier for fractional certificates. Once $\mathcal{T}$ is fixed, its leaves define the composed Boolean function $F_{\mathcal{T}}$. The composition theory of Gilmer, Saks, and Srinivasan~\citep{GSS16}, together with the equivalence between fractional and randomized certificates, recovers the same $\Omega(\rho^d)$ randomized-query lower bound in the independent-leaf case. \Cref{thm:adaptive-query-lower-bound} is the corresponding debate-verification form: it keeps the revealed tree fixed, allows arbitrary adaptive oracle queries, and gives the explicit advantage bound $m\rho^{-d}$.
\end{remark}

\else
We refer the reader to \Cref{sec:lower-bound-app} for the proof of \Cref{thm:adaptive-query-lower-bound}.
\fi

\section{Discussion and Open Problems}
Our protocol and matching lower bound show that fractional block sensitivity precisely characterizes the set of problems that can be efficiently solved by debate within this recursive-decomposition framework. Notably, the protocol does not require Bob to determine whether or not Alice's answers to subproblems are correct, but only to understand the structure of the argument well enough to sample subproblems to which the outcome is sufficiently sensitive. This stands in contrast to the original $\text{Debate}=\text{PSPACE}$ result of \cite{irving2018ai}, which requires Bob to actually directly determine which of Alice's subproblems are flawed. Of course, this comes with the limitation that the identification of flawed subproblems might actually have much greater computational complexity than creating the subproblems. This is not an issue if we assume that both Alice and Bob have the power to solve any problem in PSPACE, but is a serious obstacle if the debaters are assumed to be computationally bounded (as they will be in practice).

Given this discussion, the question remains: is it possible to design a protocol that goes beyond the fractional block sensitivity barrier we identify? One might hope that utilizing Bob's ability to solve the same problems as Alice could help with this issue, perhaps by allowing Bob to take advantage of Alice's computational limitations in some way. Alternatively, our lower bounds only apply when human judgment is modeled as a black box, accessible only via queries. Classical results in interactive proofs like $\text{IP}=\text{PSPACE}$ show that opening up the black box yields much stronger protocols. It seems implausible that these methods can be adapted to the case of human judgment---one would need to essentially be able to ``read the code'' that implements the human judgment. However, it is still possible that some other relaxation of the black-box oracle assumption could lead to better debate protocols.

\section*{Acknowledgments}
JL is supported by Scott Aaronson's Coefficient Giving grant.
ZX is supported by the NSF AI Institute for Foundations of Machine Learning (IFML).

\bibliographystyle{plainnat}
\bibliography{standard}

\end{document}